\documentclass[11pt]{article}

\usepackage[final]{acl}

\usepackage{subcaption}
\usepackage{times}
\usepackage{latexsym}
\usepackage[T1]{fontenc}

\usepackage[utf8]{inputenc}

\usepackage{microtype}

\usepackage{inconsolata}

\usepackage{graphicx}

\usepackage{amsmath}
\usepackage{amssymb}
\usepackage{amsthm}
\usepackage{booktabs}
\usepackage{dsfont}
\usepackage{listings}
\usepackage{mathtools}
\usepackage{multirow}

\newtheorem{assumption}{Assumption}
\newtheorem{theorem}{Theorem}

\title{Dense Process Supervision for Search Agents via Fact Utility Estimation}

\author{
    \textbf{Rongzhi Zhu}$^{1,*}$ \quad 
    \textbf{Xiangyu Liu}$^{1,*}$ \quad 
    \textbf{Yi Liu}$^{1}$ \quad 
    \textbf{Shuo Zhang}$^{2}$ \quad
    \textbf{Ruirui Zhang}$^{2}$ \\
    \textbf{Rui Wu}$^{2}$ \quad 
    \textbf{Tao Jiang}$^{2}$ \quad 
    \textbf{Zequn Sun}$^{1,\dagger}$ \quad 
    \textbf{Wenhao Xu}$^{2}$ \quad 
    \textbf{Wei Hu}$^{1,3}$ \\
    $^1$ State Key Laboratory for Novel Software Technology, Nanjing University, China \\
    $^2$ Ant Group, China \\
    $^3$ National Institute of Healthcare Data Science, Nanjing University, China \\
    \texttt{\{rzzhu,\,xyl,\,yiliu07\}.nju@gmail.com, \{sunzq,\,whu\}@nju.edu.cn} \\
    \texttt{\{suzhang.zs,\,zhangruirui.zrr,\,guli.wr,\,tara.jt,\,hao.xuwh\}@antgroup.com}
}

\begin{document}
\maketitle

\begingroup
\renewcommand{\thefootnote}{\fnsymbol{footnote}}
\footnotetext[1]{Equal contribution.}
\footnotetext[2]{Corresponding author.}
\endgroup
\begin{abstract}
Reinforcement learning (RL) for search agents typically relies on outcome rewards.
However, it often fails to achieve effective credit assignment, due to the unclear value of intermediate steps.
It is hard to separate their contributions from the final result.
In this paper, we propose a dense process supervision method based on fact utility estimation, which models the reasoning process as the accumulation of discrete evidence facts.
We first extract structured facts from raw observations and organize them into an explicit fact store.
To support credit assignment, we then cluster semantically equivalent facts and infer the posterior utility of each fact cluster using Bayesian estimation over group rollouts.
Finally, we convert the estimated fact utilities into dense step-level rewards to guide RL training.
Experiments on seven single-hop and multi-hop QA benchmarks show that our method consistently outperforms existing baselines.
Ablation studies validate clear relative improvements on multi-hop QA compared to outcome reward-only training.
\end{abstract}

\section{Introduction}

Large language models (LLMs) have shown strong capabilities in reasoning and language understanding. 
However, applying LLMs to complex, real-world scenarios remains a challenge \cite{gaia,hle}.
To bridge this gap, LLM agents have been introduced, which can call external tools \cite{react,toolformer}, such as programming interpreters~\cite{pal}, mathematical solvers~\cite{mathtool}, and retrieval engines~\cite{search_agent,agentorchestra}, to solve problems that pure generation cannot handle.
Notably, recent proprietary systems like Deep Research~\cite{geminideepresearch,openaideepresearch,tongyi_deepresearch} have achieved remarkable success in autonomous information gathering. 
These systems are able to quickly retrieve and analyze a large number of web resources, and generate in-depth reports for specific domains within a short time.
Such tasks require strong agentic and tool-use abilities like planning and information extraction.

Reinforcement learning (RL) has been proven to be an effective training strategy to enhance tool-use abilities~\cite{search-r1,gigpo}. 
Most existing RL methods for reasoning and search-based agents rely on outcome-based supervision, where the agent is rewarded solely according to the correctness of the final answer~\cite{research,zerosearch}.
However, this signal is not only sparse and delayed, but also lacks fine granularity. 
Crucially, the contribution of a specific tool call is often implicit and context-dependent. 
A correct final answer does not imply that every retrieved evidence was necessary, nor does a failure invalidate all prior reasoning steps. 
As a result, it is hard to determine which intermediate actions are truly helpful, limiting the training effectiveness.

To address this challenge, we propose a search agent training method, FactAgent, for dense process supervision via fact utility estimation. 
Our key idea is to represent intermediate evidence as structured facts, estimate their utility from group rollouts, and use these utilities to provide dense step-level rewards.
FactAgent incorporates a fact store to explicitly maintain these structured facts.
This design departs from standard ReAct agents~\cite{react}, which rely on raw text history.
In addition to the search and answer steps, we integrate an assert step to extract structured facts from retrieved observations and write them into the fact store.
We treat these facts as discrete semantic units that influence the trajectory's success. 
To enable credit assignment, we propose a \emph{Bayesian utility estimation} method.
We cluster semantically equivalent facts and estimate the posterior utility of each fact cluster with Bayesian estimation from group-level rollouts, capturing their statistical association with successful outcomes. 
This enables more precise credit assignment to intermediate steps that uncover high-value evidence, guiding the agent toward more effective exploration.

Our contributions are summarized as follows:
\begin{itemize}
    \item \textbf{FactAgent framework.} We introduce FactAgent, a fact-centric search agent that extracts retrieved evidence into structured facts and organizes them in an explicit fact store. 
    It replaces raw interaction history with a compact representation, effectively handling long contexts as interactions grow.
    
    \item \textbf{Bayesian-grounded dense process supervision.} 
    Motivated by the view that the utility of an intermediate fact is implicit and only reflected in final outcomes, we estimate the posterior utility of fact clusters from group-level rollout statistics. 
    This yields dense process rewards directly from sparse outcome signals, without requiring external reward models.
    
    \item \textbf{Strong performance.} 
    FactAgent consistently outperforms baselines on seven QA benchmarks with Qwen2.5-3B-Instruct and -7B-Instruct backbones.
    With the 7B backbone, it achieves an average EM of 51.2, surpassing the strongest baseline by +3.2 EM. 
    Ablations and further analyses show that our dense process supervision can particularly improve evidence precision and retrieval quality.
\end{itemize}

\section{Related Work}

\subsection{RAG and Search-Based LLM Agents}
Although LLMs are capable, their internal knowledge can be incomplete or outdated, which may cause factual errors \cite{outdated}.
Retrieval-augmented generation (RAG) \cite{rag} mitigates this issue by grounding generation on an external corpus via retrieval.
Early RAG often assumes a largely static knowledge base. 
Many works improve retrieval quality by introducing graph structures and entity-level matching \cite{graphrag,lightrag,chainrag}. 
More recently, retrieval is increasingly treated as an on-demand tool that LLMs can call when needed \cite{Tura}. 
In this tool-use setting, the LLM is equipped with multiple tools (e.g., web search and browsing) and acts as a search agent \cite{search_agent,Search-o1} that iteratively collects and synthesizes evidence, as explored in recent multi-agent frameworks \cite{agentorchestra,metagpt} and deep-research style systems \cite{openaideepresearch,geminideepresearch,step}. 
Several benchmarks evaluate multi-step retrieval and reasoning \cite{gaia,hle,browsecomp}. 
These benchmarks also show that planning and optimizing multi-step search remains challenging.
Beyond retrieval quality, prior work has also examined whether retrieved evidence appropriately and sufficiently supports the generated answer through fine-grained citation evaluation \cite{awecita}.

\begin{figure*}[th]
\begin{center}
\centerline{\includegraphics[width=0.99\linewidth]{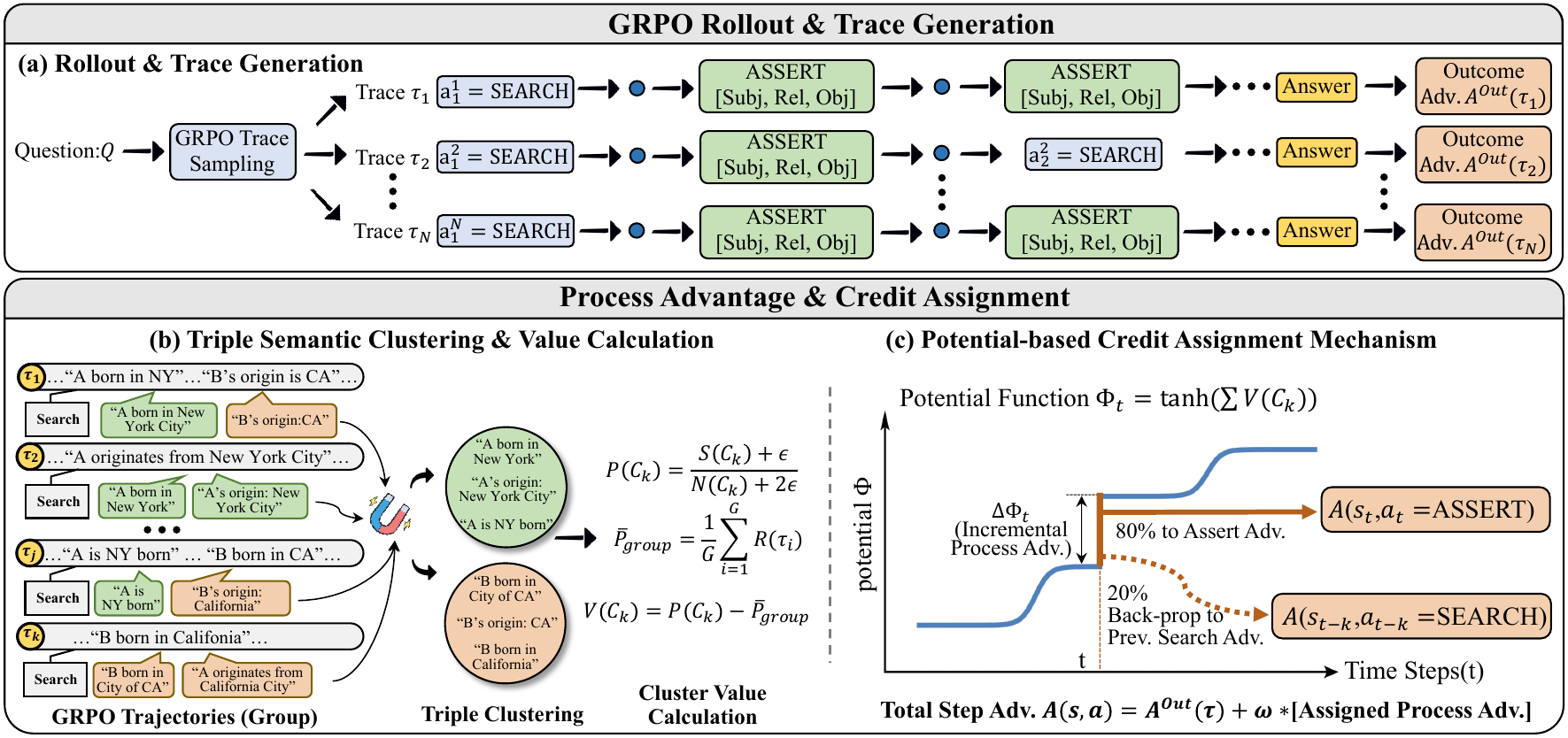}}
\caption{Framework overview.
(a) Group rollouts with Assert actions populate the fact store.
(b) Facts are clustered for Bayesian utility estimation.
(c) Utilities are converted into dense rewards for fine‑grained credit assignment.}
\label{framework}
\end{center}
\end{figure*}

\subsection{RL for Search Agents}
Recent advancements leverage RL to evolve LLMs into autonomous search agents capable of interacting with dynamic environments \cite{DeepResearcher, WebDancer, WebShaper, zerosearch}. 
Many methods jointly optimize reasoning and search. 
Search-R1 \cite{search-r1} and ReSearch \cite{research} model search actions as part of the reasoning trajectory and apply group relative policy optimization (GRPO) with sparse outcome rewards.
RL is also used to support long-horizon information seeking by improving memory management. 
MemSearcher \cite{MemSearcher} and MEM1 \cite{MEM1} maintain a compact memory state that is updated over time, while Memory-R1 \cite{Memory-R1} treats memory operations (e.g., ADD, DELETE) as learnable actions. 
ReSum \cite{ReSum} extends search beyond the context window by summarizing intermediate states.
Recent works have explored process-level supervision to tackle the credit assignment challenge. 
AutoRefine \cite{autorefine} uses retrieval-specific rewards, and GIGPO \cite{gigpo} computes step-level advantages by clustering the same intermediate states. 
In contrast, our method performs clustering over extracted factual units rather than raw agent states, deriving dense intrinsic supervision directly from rollout statistics.

\section{Preliminaries}

\textbf{Problem formulation.}
We model the multi-turn reasoning of an LLM agent as a Markov Decision Process (MDP) $\langle \mathcal{S}, \mathcal{A}, \mathcal{P}, \mathcal{R} \rangle$.
Given an original question $Q$, at step $t$ the agent observes $o_t$ (e.g., retrieved text) and maintains an explicit knowledge state $\mathcal{F}_t$.
We define the state $s_t = (Q, o_t, \mathcal{F}_t)$.
Unlike other agents that primarily rely on unstructured history as implicit memory, we use $\mathcal{F}_t$ as a compact and interpretable representation of accumulated evidence updated throughout the trajectory.

\smallskip\noindent\textbf{Fact store representation.}
Each entry in $\mathcal{F}_t$ consists of structured facts derived from observations, paired with a short description to summarize the source and reliability. 
We write $e \in \mathcal{F}_t$ for a semantic fact unit, and denote the newly asserted facts at step $t$
by $\mathcal{K}_t$. 
As detailed in Section \ref{method}, we treat these units not merely as text strings, but as discrete semantic variables  that serve as explicit units of evidence extracted from the observation.

\smallskip\noindent\textbf{Actions and transition.}
The action space $\mathcal{A}$ contains three actions:
(i) \textbf{search}, which generates a query and receives a new observation;
(ii) \textbf{assert}, which distills the unstructured observation into structured triples and updates the explicit fact store;
and (iii) \textbf{answer}, which terminates the episode and outputs the final response.
The state transition $\mathcal{P}$ is defined mainly based on the symbolic update of the fact store:
after an assert action, the fact store updates as $\mathcal{F}_{t+1}=\mathcal{F}_t \cup \mathcal{K}_t$, representing the acquisition of new semantic evidence units.

\smallskip\noindent\textbf{Training.}
The objective is to maximize expected outcome rewards: 
$
\max_{\pi} \mathbb{E}_{\tau \sim \pi}\left[ R_{\text{outcome}}(\tau) \right].
$
In practice, $R_{\text{outcome}}$ is sparse and only observed at the end of trajectories,
making direct optimization difficult.
As a result, recent search agents adopt GRPO to optimize a surrogate objective defined over groups of sampled trajectories.
Specifically, given a group of rollouts $\{\tau_i\}_{i=1}^G$ generated under the current policy,
GRPO updates the policy by maximizing a relative advantage objective:
\begin{align}
\resizebox{0.88\columnwidth}{!}{$
\mathcal{L}_{\text{GRPO}}(\pi)
= \mathbb{E}_{q \sim \mathcal{D}, \{\tau_i\} \sim \pi} \left[
    \frac{1}{G} \sum\limits_{i=1}^G \frac{1}{|\tau_i|} \sum\limits_{t} \frac{\pi\left(a_t^i | s_t^i\right)}{\pi_\text{old}\left(a_t^i | s_t^i\right)} A_i
\right],
$}
\end{align}
where $A_i = \frac{R(\tau_i) - \mu}{\sigma}$ is the advantage normalized by the group mean $\mu$ and standard deviation $\sigma$.

While outcome supervision provides a global signal, it obscures individual contributions, making it difficult to identify necessary steps.
To address this limitation, we introduce dense process supervision by constructing informative step-level signals for credit assignment and advantage estimation under GRPO, as described in Section~\ref{method}.

\section{Methodology}
\label{method}

\subsection{FactAgent Framework}

To support dense supervision and credit assignment, we propose FactAgent.
Standard ReAct agents rely on the full interaction history, which leads to rapid context growth and limits reasoning depth.
Instead of using such raw text history, 
we maintain a persistent fact store $\mathcal{F}$ as the working memory. 
During interaction, the agent sees the original question $Q$, the current fact store $\mathcal{F}$, and only the most recent observation $o_t$. 
This keeps the prompt length approximately constant across steps while allowing the agent to reason over structured evidence.
At each step $t$, the agent observes the current state and selects one action from the following three types:
\begin{enumerate}
    \item \textbf{Search (information seeking).}
    When external information is needed, the agent issues a query.
    The environment returns an observation $o_t$.
    Crucially, we overwrite the previous observation in the prompt with $o_t$. 
    This ensures the input context size remains constant regardless of the trajectory length.

    \item \textbf{Assert (information distillation).}
    Given $o_t$, the agent extracts a small set of structured triples and adds them into the fact store as $\mathcal{K}_t$.
    Each $\mathcal{K}_t$ also includes a short textual description to verbalize the triples.

    \item \textbf{Answer (termination).}
    Once the information is enough to answer the question, the agent produces the final answer
    based on $Q$ and $\mathcal{F}$, without relying on long raw histories.
\end{enumerate}

From the lens of partial information, $\mathcal{F}_t$ provides an explicit and compact representation of the evidence accumulated so far.
We formally prove in Appendix~\ref{appx:markov_equivalence} that the fact store serves as a sufficient statistic of the interaction history, ensuring the Markov property of the fact-centric state transitions and the preservation of policy optimality.
This structure not only  allows the agent to perform multi-step reasoning efficiently under strict context length constraints, but also enables intermediate supervision on fact store updates.

\subsection{Dense Process Rewards via Fact Utility Estimation}

We propose a strategy for deriving dense process-level rewards from sparse outcome signals under the GRPO setting.
To tackle the credit assignment challenge in complex reasoning, we reframe the problem as a Bayesian estimation task.
Specifically, we aggregate semantically related reasoning facts to enable stable advantage estimation.
The resulting cluster-level fact utilities are then transformed into dense process rewards via potential-based shaping, ensuring policy invariance while providing fine-grained supervision.

\smallskip\noindent\textbf{Reasoning trajectories and semantic facts.}
Given a query $q$, a policy $\pi_\theta$ generates a reasoning trajectory
$
\tau = (a_1, s_1), (a_2, s_2), \dots, (a_T, s_T),
$
which terminates with a binary outcome reward
$R(\tau) \in \{0,1\}$ indicating the correctness of the final answer.
We assume that each trajectory induces a set of discrete semantic facts
extracted from the agent's explicit belief state (fact store):
$
\mathcal{F}(\tau) = \{f_1, f_2, \dots, f_L\},
$
where each $f_j$ represents a fact.
These facts serve as a structured summary of the agent's intermediate knowledge.

The optimization objective is to maximize the expected terminal reward
$\mathbb{E}_{\tau \sim \pi_\theta}[R(\tau)]$.
However, this terminal-only signal provides no direct supervision for
intermediate decisions, leading to severe credit assignment challenges.
Ideally, one would like to quantify the contribution of each fact $f$ via
\begin{align}
v(f) \approx \Pr(R=1 \,|\, f \in \mathcal{F}(\tau)),
\end{align}
but such estimation is infeasible in practice due to the extreme sparsity
and high dimensionality of the space of natural language facts.
In typical GRPO group rollouts, most individual facts appear only once or not at all,
making fact-level empirical estimates statistically unreliable.

\smallskip\noindent\textbf{Semantic abstraction of reasoning facts.}
To address the sparsity issue, we aggregate semantically equivalent facts into clusters.
Let $\mathcal{C} = \{C_1, \dots, C_M\}$ denote a partition of the discovered fact space,
where each cluster represents a distinct semantic concept (e.g., grouping ``A was born in B'' and ``A originates from B'').
We employ a two-stage strategy to determine if two facts are semantically equivalent ($f_i \sim f_j$):
first via coarse-grained embedding similarity, followed by fine-grained logical rule verification which enforces constraints such as negation consistency, exact numeric matching, and relation-level similarity, to prevent spurious merges  (detailed in Appendix \ref{appx:cluster}).
Formally, each cluster $C_k$ is defined as an equivalence class:
\begin{align}
C_k = \{ f \in \mathcal{F(\tau)} \,|\, f \sim f_{\text{rep}} \},
\end{align}
where $f_{\text{rep}} \in C_k$ is a representative fact.
Instead of estimating utility for a rare isolated fact $f$, we estimate it for the semantic cluster $C_k$ containing $f$.
This shares statistical strength across facts.

\smallskip\noindent\textbf{Bayesian estimation of cluster-level fact utility.}
We aim to estimate the utility of each cluster $C_k$, denoted as a latent success probability $\theta_k$.
Intuitively, $\theta_k$ represents the likelihood that a trajectory leads to a correct answer, given that its extracted facts overlap with the semantic concept $C_k$:
\begin{align}
\theta_k = \Pr(R=1 \,|\, \mathcal{F}(\tau) \cap C_k \neq \emptyset).
\end{align}

However, in GRPO, the group size $G$ is typically small (e.g., $G \in [4, 16]$).
In such small-sample regimes, simply calculating the empirical win rate yields high-variance estimates, where a single coincidental success could skew the probability to extreme values. 
To enable robust estimation under this data sparsity, we adopt a Bayesian formulation by placing a symmetric Beta prior on each cluster
\begin{align}
\theta_k \sim \text{Beta}(\epsilon, \epsilon),
\end{align}
where $\epsilon > 0$ controls prior smoothing.

Given a GRPO group of trajectories,
let $N(C_k)$ denote the number of trajectories where cluster $C_k$ is present,
and $S(C_k)$ the number of those trajectories with $R=1$.
Under a Bernoulli likelihood for $R$ conditioned on the presence of $C_k$,
the posterior distribution of $\theta_k$ is
\begin{align}
\resizebox{\columnwidth}{!}{$
\theta_k \,|\, S(C_k), N(C_k)
\sim
\text{Beta}(S(C_k)+\epsilon,\, N(C_k)-S(C_k)+\epsilon).
$}
\end{align}

We use the posterior mean as a stable estimate:
\begin{align}
\resizebox{0.88\columnwidth}{!}{$
P(C_k)
= \mathbb{E}\!\left[\theta_k \,|\, S(C_k), N(C_k)\right]
= \frac{S(C_k)+\epsilon}{N(C_k)+2\epsilon}.
$}
\end{align}

\smallskip\noindent\textbf{Relative utility and GRPO advantage approximation.}
In GRPO, optimization depends on relative performance within a sampled group
rather than absolute success probabilities.
Let
\begin{align}
\bar{P}_{\text{group}} = \frac{1}{G}\sum_{i=1}^G R(\tau_i)
\end{align}
denote the empirical success rate of the current group.
We define the relative utility of cluster $C_k$:
\begin{align}
\label{eq:relative_utility}
V(C_k) = P(C_k) - \bar{P}_{\text{group}}.
\end{align}

This quantity represents the posterior gain over the group baseline, serving as a cluster-level proxy for the advantage function.
A positive value indicates that the presence of $C_k$ is associated with
above-average success, while a negative value suggests correlation with failure.
Centering by the group mean also stabilizes optimization and aligns naturally
with GRPO’s relative policy updates.

\smallskip\noindent\textbf{Process reward via potential-based shaping.}
To transform static cluster utilities into dense, step-wise feedback,
we adopt Potential-Based Reward Shaping (PBRS),
which preserves policy optimality.
Let $\mathcal{U}_t = \{ C_k \in \mathcal{C} \,|\, C_k \cap \mathcal{F}(s_t) \neq \emptyset \}$ denote the set of unique semantic clusters covered by the facts in the fact store at step $t$.
The potential function over partial states is defined as
\begin{align}
\label{eq:potential}
\Phi(s_t) = \tanh \Big( \sum_{C \in \mathcal{U}_t} V(C) \Big).
\end{align}
Using only unique clusters prevents reward hacking through redundant
or paraphrased fact assertions.
The $\tanh$ nonlinearity bounds the potential, improving training stability.

The raw shape reward at step $t$ is then
\begin{align}
r^{\text{shape}}_t = \Delta \Phi_t = \Phi(s_t) - \Phi(s_{t-1}).
\end{align}

This reward is non-zero only when new semantic information is acquired,
providing immediate feedback for informative steps while staying consistent with the original sparse-reward objective.

\smallskip\noindent\textbf{Credit assignment between search and assert actions.}
Although potential gains materialize at assert steps, the correctness of an assertion is causally dependent on prior search actions that retrieved relevant evidence.
To address this temporal lag, we redistribute the reward.
Let $t$ be an assert step and $t' < t$ be the specific \textsc{search} step that yielded the asserted fact.
Let $\alpha \in [0,1]$ denote the attribution coefficient.
The process rewards are assigned as
\begin{align}
r^{\text{proc}}_t &= (1-\alpha)\,r^{\text{shape}}_t, \ \
r^{\text{proc}}_{t'} &= \alpha\,r^{\text{shape}}_t.
\end{align}

This mechanism explicitly encourages the policy to generate high-quality search queries that lead to valuable knowledge acquisition.

\smallskip\noindent\textbf{Auxiliary process penalties.}
Finally, we add auxiliary step-level penalties to discourage
invalid or degenerate behaviors, including:
(1) format violations,
(2) repetitive searches, and
(3) premature answers with an empty fact store.
In the last case, we additionally enforce a hard constraint by setting the terminal reward $R(\tau)=0$, ensuring that correct answers must be grounded in explicit evidence.

\subsection{RL Training}
We employ GRPO as our training backbone. Unlike standard proximal policy methods that require a separate value network, GRPO estimates the baseline using the group mean, significantly reducing memory overhead. 
Unlike ReAct agents that optimize over a single growing prompt, our agent operates on step-specific state contexts constructed from the current fact store. 
Consequently, we perform optimization at the action level: each decision $a_{i,j}$ is conditioned on a compact state $s_{i,j}$ and is updated via policy gradients weighted by its specific hybrid advantage $A^{\text{total}}_{i,j}$.

\smallskip\noindent\textbf{Advantage composition.}
We first compute the outcome advantage for each trajectory. Let $R_{\text{outcome}}(\tau_i)$ be the reward derived from the final result of trajectory $\tau_i$. 
We normalize this reward against the group statistics to obtain the outcome advantage $A^{\text{outcome}}_i$:
\begin{align}
\resizebox{\columnwidth}{!}{$
A^{\text{outcome}}_i = \frac{R_{\text{outcome}}(\tau_i) - \text{Mean}(\{R_{\text{outcome}}(\tau_k)\}_{k=1}^G)}{\text{Std}(\{R_{\text{outcome}}(\tau_k)\}_{k=1}^G) + \eta}.
$}
\end{align}

This global signal is then broadcast to all actions within the trajectory.
The total advantage for the $j$-th action in the $i$-th trajectory, denoted as $A^{\text{total}}_{i,j}$, is computed by augmenting this global signal with the process advantage from the previous section.
Specifically, we directly utilize the redistributed process reward $r^{\text{proc}}_{i,j}$ as the process advantage term:
\begin{align}
    A^{\text{total}}_{i,j} = A^{\text{outcome}}_i + \Omega \cdot r^{\text{proc}}_{i,j},
\end{align}
where $\Omega$ is a coefficient balancing the impact of local process supervision against the global objective.
We use $\Omega=0.5$ by default and report a sensitivity study in Appendix~\ref{appx:hyperparameter}.
Since $r^{\text{proc}}_{i,j}$ represents a difference in value potentials ($\Delta \Phi$), it intrinsically serves as a proxy for the step-wise value improvement, justifying its use as a local advantage signal.

\smallskip\noindent\textbf{Optimization objective.}
Since our trajectories consist of multiple discrete steps under iteratively updating fact store contexts, we define the optimization objective by aggregating gradients over all individual actions across the group. 
Let $n_i$ be the number of steps in trajectory $\tau_i$, and $s_{i,j}$ denote the state at step $j$ (comprising the user query and the current fact store).
The policy $\pi_\theta$ is optimized by maximizing the following objective:
\begin{equation}
\resizebox{.86\columnwidth}{!}{$
\begin{split}
    \mathcal{J}(\theta) & = \mathbb{E}_{q \sim \mathcal{D}, \{\tau_i\} \sim \pi_{\theta_{\text{old}}}} \bigg[ \frac{1}{\sum_{i} n_i} \sum_{i=1}^G \sum_{j=1}^{n_i} \Big( \mathcal{L}^{\text{clip}}_{i,j}(\theta) \\
    &\quad - \beta \mathbb{D}_{KL}\big(\pi_\theta(\cdot \,|\, s_{i,j}) \parallel \pi_{\text{ref}}(\cdot \,|\, s_{i,j})\big) \Big) \bigg],
\end{split}
$}
\end{equation}
where $\beta$ controls the KL divergence penalty to prevent the policy from deviating excessively from the reference model $\pi_{\text{ref}}$. 
The step-wise clipped surrogate loss, denoted by $\mathcal{L}^{\text{clip}}_{i,j}(\theta)$, is defined as:
\begin{align}
\resizebox{\columnwidth}{!}{$
\mathcal{L}^{\text{clip}}_{i,j}(\theta) = \min \left( \rho_{i,j} A^{\text{total}}_{i,j}, \text{clip}(\rho_{i,j}, 1-\epsilon_{\text{clip}}, 1+\epsilon_{\text{clip}}) A^{\text{total}}_{i,j} \right),
$}
\end{align}
where $\rho_{i,j} = \frac{\pi_\theta(a_{i,j}|s_{i,j})}{\pi_{\theta_{\text{old}}}(a_{i,j}|s_{i,j})}$ denotes the importance sampling ratio. 
This formulation ensures that every reasoning step contributes to the policy update, weighted by its specific hybrid advantage $A^{\text{total}}_{i,j}$.

\begin{table*}[t]
\centering
\resizebox{\linewidth}{!}{
\begin{tabular}{c l ccc cccc c}
\toprule
\multirow{2}{*}{\textbf{Backbone}} & \multirow{2}{*}{\textbf{Method}} &
\multicolumn{3}{c}{\textbf{Single-Hop QA}} &
\multicolumn{4}{c}{\textbf{Multi-Hop QA}} &
\multirow{2}{*}{\textbf{Avg.}} \\
\cmidrule(lr){3-5}\cmidrule(lr){6-9}
& & \textbf{NQ} & \textbf{TriviaQA} & \textbf{PopQA} & \textbf{HotpotQA} & \textbf{2Wiki} & \textbf{MuSiQue} & \textbf{Bamboogle} & \\
\midrule
\parbox[t]{2mm}{\multirow{8}{*}{{{{3B}}}}}
& RAG          & 34.8 & 54.4 & 38.7 & 25.5 & 22.6 & \phantom{0}4.7 & \phantom{0}8.0 & 34.4 \\
& Search-R1    & 34.1 & 54.5 & 37.8 & 32.4 & 31.9 & 10.3 & 26.4 & 37.7 \\
& Zerosearch   & 41.4 & 57.4 & \textbf{44.8} & 27.4 & 30.0 & \phantom{0}9.8 & 11.1 & 39.5 \\
& GiGPO        & 42.0 & 59.5 & 42.4 & 36.9 & 37.0 & 12.6 & \textbf{64.1} & 42.7 \\
& ReSearch     & 20.4 & 33.5 & 17.3 & 35.6 & 39.3 & \textbf{17.3} & \underline{37.6} & 29.1 \\
& AutoRefine   & \underline{43.6} & \underline{59.7} & \underline{44.7} & \underline{40.4} & 38.0 & \underline{16.9} & 33.6 & \underline{44.3} \\
\cmidrule{2-10}
& FactAgent (w/o RL) & 22.1 & 41.1 & 20.1 & 26.1 & 23.3 & \phantom{0}4.3 & 18.4 & 25.7 \\
& FactAgent (w/ RL)  & \textbf{43.9} & \textbf{60.3} & 42.3 & \textbf{42.6} & \textbf{43.1} & \underline{16.9} & 32.8 & \textbf{45.4} \\
\midrule
\parbox[t]{2mm}{\multirow{7}{*}{{{7B}}}}
& RAG          & 34.9 & 58.5 & 39.2 & 29.9 & 23.5 & \phantom{0}5.8  & 20.8 & 36.4 \\
& Search-R1    & 39.3 & 61.0 & 39.7 & 37.0 & 41.4 & 14.6 & 36.8 & 43.2 \\
& Zerosearch   & 43.6 & \underline{65.2} & \textbf{48.8} & 34.6 & 35.2 & 18.4 & 27.8 & 45.2 \\
& GiGPO        & \textbf{46.4} & 64.7 & 46.1 & 41.6 & 43.6 & 18.9 & \textbf{68.9} & 47.7 \\
& ReSearch     & 40.9 & 63.7 & 44.6 & \underline{43.5} & \underline{47.6} & \textbf{22.3} & 42.4 & \underline{48.0} \\
\cmidrule{2-10}
& FactAgent (w/o RL) & 33.0 & 55.1 & 36.7 & 29.2 & 23.8 & \phantom{0}9.1  & 20.8 & 34.9 \\
& FactAgent (w/ RL)  & \underline{46.3} & \textbf{69.4} & \underline{46.7} & \textbf{44.6} & \textbf{51.5} & \underline{19.4} & \underline{44.0} & \textbf{51.2} \\
\bottomrule
\end{tabular}}
\caption{Performance comparison on single-hop and multi-hop QA datasets with Qwen2.5-Instruct backbones.}
\label{tab:main_results}
\end{table*}

\section{Experiment}
\subsection{Experiment Setups}

\textbf{Datasets and metrics.} 
Following Search-R1 \cite{search-r1}, 
we evaluate our method on seven QA benchmarks, covering both single-hop and multi-hop settings. 
The single-hop QA datasets include Natural Questions (NQ)~\cite{nq}, TriviaQA~\cite{triviaqa}, and PopQA~\cite{popqa}.
The multi-hop QA datasets include HotpotQA~\cite{hotpotqa}, 2WikiMultiHopQA (2Wiki)~\cite{2wiki}, MuSiQue~\cite{musique}, and Bamboogle~\cite{bamboogle}. 
More details can refer to Appendix~\ref{appx:dataset}.
We report the weighted average EM over all datasets, computed as the total number of exact matches divided by the total number of questions.

\smallskip\noindent\textbf{Baselines.} 
To isolate the effect of RL optimization, we evaluate our method before RL training and compare our method against two groups of baselines: 
(1) retrieval-based methods without RL training, such as Naive RAG~\cite{rag}; 
and (2) retrieval-based methods with RL training, including Search-R1~\cite{search-r1}, ReSearch~\cite{research}, AutoRefine~\cite{autorefine}, GiGPO~\cite{gigpo}, and Zerosearch~\cite{zerosearch}. 
Notably, Zerosearch and ReSearch rely on Google Search, while all other methods (including ours) retrieve from a local Wiki-18 index following Search-R1.

\smallskip\noindent\textbf{Implementation.} 
We train and evaluate our method on Qwen2.5-7B-Instruct and Qwen2.5-3B-Instruct. 
We adopt the E5 embedding model~\cite{e5} as the dense retriever. 
We perform RL training based on verl \cite{verl}. 
Following Search-R1, we construct the training set by merging the training splits of NQ and HotpotQA. 
We set the smoothing parameter $\epsilon=0.5$, the return coefficient $\alpha=0.2$, and the aggregation weight $\Omega=0.5$. 
The agent is allowed to run at most 8 rounds per question (including the answer action). The outcome reward is defined as the answer-level F1 score.
Additional details are in Appendix~\ref{appx:experiment}.

\subsection{Main Results}
Table~\ref{tab:main_results} reports the EM results on seven QA benchmarks.
Overall, our method with RL achieves the best weighted average performance among all baselines on both Qwen2.5-3B-Instruct and Qwen2.5-7B-Instruct.
With the 7B backbone, our method surpasses the strongest baseline by \textbf{+3.2 EM} on average.
Notably, the performance gap is more significant on multi-hop datasets (e.g., HotpotQA, 2Wiki) than on single-hop tasks. 
This confirms that our dense process supervision effectively guides the LLM through complex reasoning chains, where sparse outcome rewards often fail.
We observe similar improvements on the 3B backbone, indicating that our method is effective across different model sizes and does not require strong backbone.

It is also worth noting that two strong baselines, Zerosearch and ReSearch, rely on Google Search to access real-time web information, whereas FactAgent retrieves exclusively from a local Wiki-18 index.
Despite operating under this more restricted retrieval setting, our approach still achieves superior overall performance, suggesting that it can more effectively identify and utilize relevant information from limited retrieval results.

We also evaluate our method before RL training to assess the impact of RL optimization.
As shown in Table~\ref{tab:main_results}, the model without RL performs poorly and even underperforms a simple one-shot RAG baseline, particularly on the 3B backbone.
This is expected, as the LLM has not yet adapted to the specialized ``Search--Assert--Answer'' workflow.
After RL training, performance improves substantially and becomes state-of-the-art.
This highlights that RL is essential for grounding the agent's behavior in the proposed fact-centric reasoning process.

Since Qwen2.5-3B-Instruct and Qwen2.5-7B-Instruct exhibit similar trends, we conduct analysis on Qwen2.5-7B-Instruct for clarity and brevity.

\begin{figure*}[t]
    \centering
    \begin{subfigure}[t]{0.32\textwidth}
        \centering
        \includegraphics[width=\linewidth]{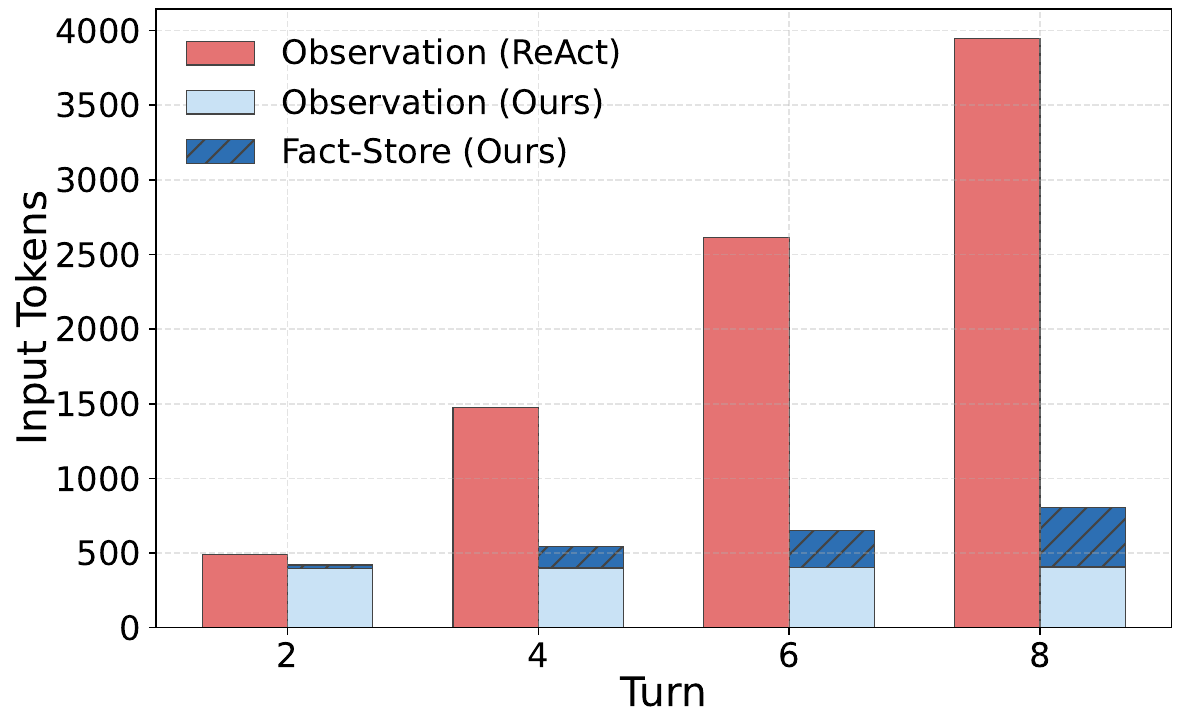}
        \caption{Input token composition across turns.}
        \label{fig:analysis-efficiency}
    \end{subfigure}
    \hfill
        \begin{subfigure}[t]{0.32\textwidth}
        \centering
        \includegraphics[width=\linewidth]{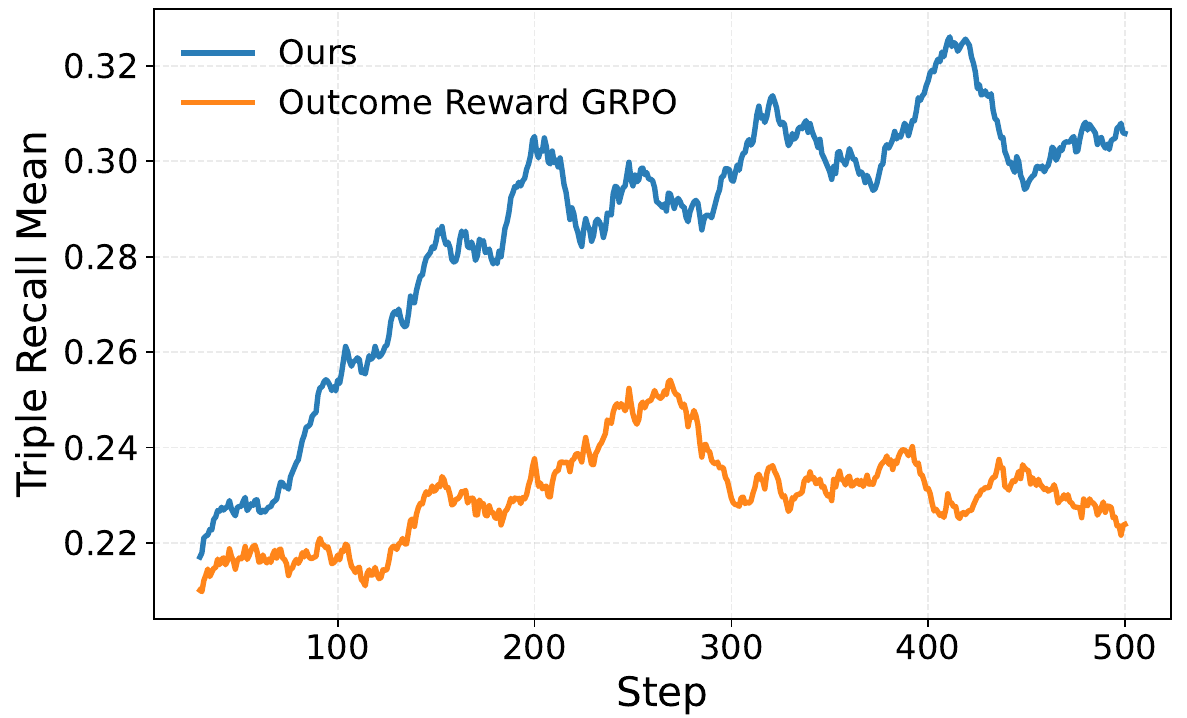}
        \caption{Fact coverage over training.}
        \label{fig:analysis-coverage}
    \end{subfigure}
    \hfill
    \begin{subfigure}[t]{0.32\textwidth}
        \centering
        \includegraphics[width=\linewidth]{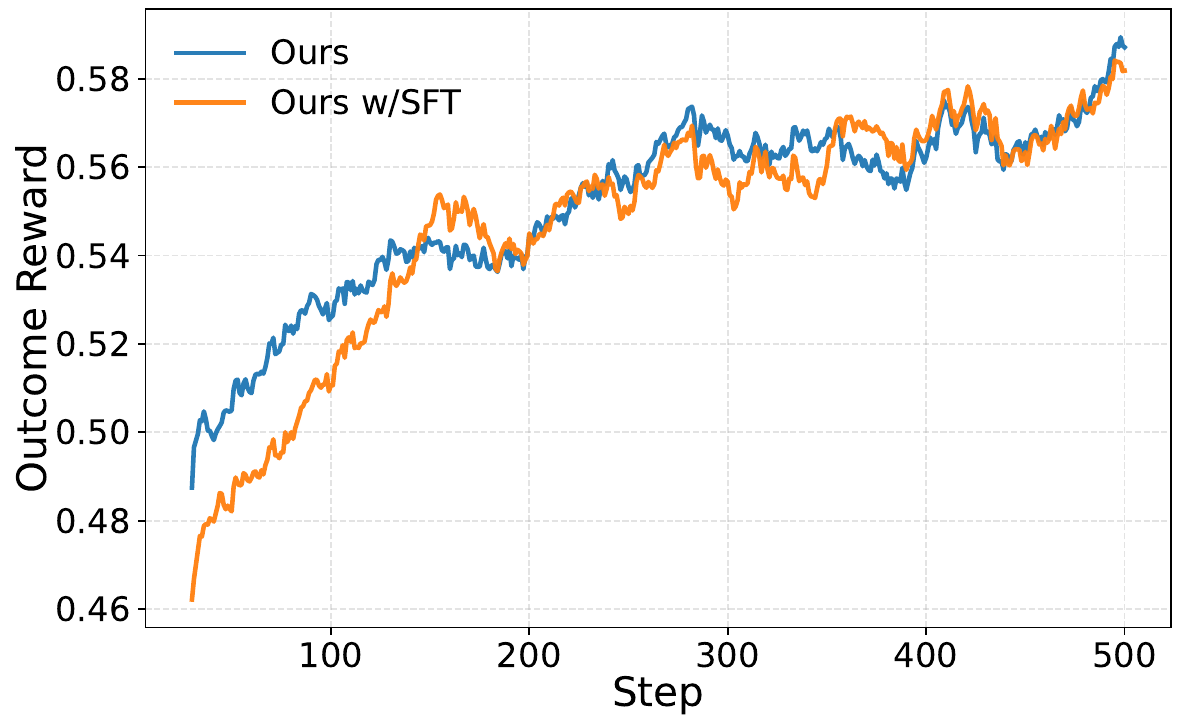}
        \caption{Cold-start RL training dynamics.}
        \label{fig:analysis-coldstart}
    \end{subfigure}
    \caption{Analysis of dense process supervision.
    (a) Input token composition comparison between ReAct and FactAgent.
    (b) Fact coverage measured by the recall of reference triples improves over training.
    (c) Cold-start training curves show faster and more stable learning with dense process supervision.}
    \label{fig:analysis}
\end{figure*}

\subsection{Analysis}

\textbf{Ablation study.}
We conduct an ablation study to examine the effect of several key design choices, including back-propagation, dense process reward, and EM-only clustering. 
As shown in Table~\ref{tab:ablation}, removing these components consistently hurts performance, especially on multi-hop datasets.
    
\begin{table}[!t]
\centering
\resizebox{\columnwidth}{!}{
\begin{tabular}{lccc}
\toprule
\textbf{Variant} & \textbf{Single-Hop} & \textbf{Multi-Hop} & \textbf{Avg.} \\
\midrule
FactAgent (w/ RL) & 55.5 & 45.8 & 51.2 \\
\ \ w/ EM-only clustering & 52.4 & 44.4 & 48.9 \\
\ \ w/o reward redistribution & 50.6 & 38.7 & 45.4 \\
\ \ w/o dense process reward & 47.5 & 31.4 & 40.5 \\
\bottomrule
\end{tabular}}
\caption{Ablation study on Qwen2.5-7B-Instruct.}
\label{tab:ablation}
\end{table}

We first consider the effect of fact clustering.
When restricting clustering to exact-match facts, performance decreases across all settings, with the average EM dropping from 51.2 to 48.9, indicating the importance of our clustering design across both single-hop and multi-hop tasks.
Due to the inherent diversity in model-generated facts, semantically equivalent facts may differ in surface form, making exact-match clustering insufficient for grouping them and leading to incomplete fact aggregation and weaker utility estimation.

Next, we remove reward redistribution and observe a clear performance decrease on.
The overall average drops from 51.2 to 45.4, with a larger reduction on multi-hop datasets than on single-hop datasets.
This indicates that propagating fact-level utility signals back to earlier retrieval steps is important for learning effective search strategies, especially when multiple evidence must be combined.

Finally, we remove dense process-level rewards and train the agent using only the final outcome reward under GRPO.
This results in the largest performance degradation among all ablations.
The results demonstrate that dense process supervision is a central component of our method.
It provides essential credit assignment signals that cannot be recovered from outcome-only rewards, especially under settings with restricted input and retrieval budgets.
Without dense process rewards, the agent fails to learn effective intermediate behaviors, particularly for challenging multi-hop queries.

\smallskip\noindent\textbf{Efficiency analysis.}
To analyze input efficiency, we evaluate the trained model under the standard ReAct framework and record the average input length at different turns.
We compare this with FactAgent under the same maximum number of actions.
As shown in Figure \ref{fig:analysis-efficiency}, ReAct accumulates the full interaction history in the prompt, leading to an approximately linear growth in input length.
At the 8th turn, the input length under ReAct is nearly four times larger than that of FactAgent.
In contrast, our method maintains a compact input by storing only extracted evidence in the fact store.

This efficiency offers two key advantages.
First, it significantly reduces the computational cost of inference, as the LLM processes far fewer tokens at later steps.
Second, it allows the agent to conduct extensive multi-turn searches without quickly reaching the context window limit of the backbone.
This demonstrates that our fact store design is highly scalable for complex reasoning tasks that require multiple rounds of information gathering.

\smallskip\noindent\textbf{Fact coverage over training.}
To verify whether our method indeed extracts useful facts, we conduct an analysis on the HotpotQA distractor setting. 
This variant provides supporting contexts relevant to the answer. 
We use GPT-OSS-120B to extract reference triples from the provided contexts. 
We then track the recall of reference facts within the agent's fact store throughout the training process.
As shown in Figure~\ref{fig:analysis-coverage}, the two methods exhibit distinct learning dynamics.
Our method demonstrates a steady improvement in fact coverage, increasing the recall by nearly 40\% relative to the initial stage.
This indicates that the dense process rewards successfully reinforce the specific action of ``asserting'' relevant information.
In contrast, the outcome-only GRPO baseline fails to show a clear upward trend.
This reveals the ambiguity of sparse outcome rewards.
Since the reward depends solely on the final answer, the agent gets the same signal whether it actively stores facts or skips the step.
Hence, the agent fails to establish a causal link between the assert action and the final success.

\subsection{Effect of SFT Cold Start}

Supervised fine-tuning (SFT) cold start is commonly used before RL to bootstrap basic capabilities and enforce output formats, especially when the action space involves structured tool calls.
In our setting, however, we find that SFT cold start does not improve the final performance ceiling. 
Although it changes early training dynamics, both cold-started and direct RL reach comparable final EM (Figure~\ref{fig:analysis-coldstart}).
Meanwhile, direct RL from the base model can occasionally collapse at certain training steps, after which the LLM outputs garbled and non-executable JSON actions. 
We never observe this failure when training from the cold-started policy.
Overall, these results suggest that cold start mainly serves as a stabilizer that anchors the policy to valid Search--Assert--Answer formats, rather than improving final task performance.
Detailed setup and results are provided in Appendix~\ref{appx:results}.

\section{Conclusion and Future Work}

In this paper, we address the credit assignment challenge for search-based agents by proposing a dense process supervision method.
Instead of treating history as unstructured text, our method organizes retrieved evidence into an explicit fact store, clusters semantically equivalent facts via Bayesian estimation, and converts these utilities into step-level rewards.
Experiments show that FactAgent consistently outperforms baselines, with significant gains on multi-hop reasoning tasks.
Future work can extend dense process supervision to richer fact representations and continuous utility signals.

\section*{Ethical Considerations}
The datasets and LLMs used in this work are public with permissive licenses.

\section*{Limitations}
Despite these advancements, our method has several limitations.
First, recovering interaction contexts to compute step-level updates increases GPU memory consumption compared to outcome-only baselines.
Second, the accuracy of our utility estimation relies on the quality of the clustering algorithm and the sufficiency of GRPO rollouts.
Finally, our evaluation is currently confined to static QA datasets; extending our FactAgent to dynamic, open-ended agentic tasks, e.g., GAIA~\cite{gaia}, remains a direction for future work.

\section*{Acknowledgments}
This work was supported by National Natural Science Foundation of China (Nos.~62676187 and 62406136), Natural Science Foundation of Jiangsu Province (No. BK20241246), and Ant Group.

\bibliography{custom}


\appendix

\section{Theoretical Analysis}
\label{appx:markov_equivalence}

In this section, we provide a theoretical justification for replacing the raw interaction history with the structured fact store. 
We distinguish between the \emph{environment feedback} (search results) and the \emph{agent state} used for decision-making. 
We show that under the assumption of semantic sufficiency, the fact store constitutes a valid state abstraction of the full history, preserving the Markov property and the optimality of policies.

\subsection{Problem Definitions}

Let the search task be modeled as a sequential decision process over $T$ steps.
\begin{itemize}
    \item \textbf{Action ($a_t$):} The command issued by the agent after observing $o_{t-1}$, either a Search action, an Assert action, or an Answer action that terminates the episode.
    We set $o_0=\varnothing$.
    \item \textbf{Environment feedback ($o_t$):} 
    The raw textual output returned by the search engine/environment in response to the most recent Search action (e.g., search snippets). In the context of search agents, this is often termed an ``observation'', but structurally it serves as input to the state transition. If $a_t$ is not Search, no new feedback is generated and we set $o_{t}=o_{t-1}$.
    \item \textbf{History ($h_t$):} The complete sequence of interactions up to time $t$:
    \begin{equation}
        h_t = (Q, o_0, a_1, o_1, a_2, o_2, \dots, a_t, o_t).
    \end{equation}
    
\end{itemize}

The baseline is the history-based MDP.
Standard LLM agents typically treat the full history $h_t$ as the state $s^{\text{hist}}_t$. Since $h_t$ contains all historical information available to the agent, the process defined over $h_t$ satisfies the Markov property: $\Pr(h_{t+1} | h_{0:t},a) = \Pr(h_{t+1} | h_t,a)$.

\subsection{Proposed: The FactAgent MDP}

We propose a state representation that discards raw history in favor of a structured fact store.
Let the proposed state be:
\begin{equation}
    s^{\text{fact}}_t = (Q, o_t, \mathcal{F}_t),
\end{equation}
where $\mathcal{F}_t$ is the set of accumulated facts. The state transition is governed by a deterministic update function $\mathcal{T}$:
\begin{equation}
    \mathcal{F}_{t+1} = \mathcal{T}(\mathcal{F}_{t}, o_{t}, a_{t+1}),
\end{equation}
where $a_{t+1}$ is chosen after observing $o_{t}$.
In particular, $\mathcal{T}(\mathcal{F}, o, a)=\mathcal{F}$ if $a$ is not Assert.

\subsection{Equivalence Analysis}

To prove that the FactAgent is not an approximation but an equivalent formulation, we introduce the \emph{Semantic Sufficiency Assumption}.

\begin{assumption}[semantic sufficiency]
\label{asm:sufficiency}
The fact store $\mathcal{F}_t$ is a \emph{sufficient statistic} of the history $h_t$ for predicting future environment feedback and final rewards. Formally, let $\Phi$ be the mapping such that $s^{\text{fact}}_t = \Phi(h_t)$. For any two histories $h_t, h_t'$ such that $\Phi(h_t) = \Phi(h_t')$, we assume:
\begin{align}
\resizebox{.87\columnwidth}{!}{$
\begin{aligned}
    \Pr(o_{t+1} \,|\, h_t, a_{t+1}) &= \Pr(o_{t+1} \,|\, \Phi(h_t), a_{t+1}), \\
    \Pr(R \,|\, h_t, a_{t+1}) &= \Pr(R \,|\, \Phi(h_t), a_{t+1}).
\end{aligned}
$}
\end{align}
\end{assumption}
For non-Search actions, the feedback is deterministic (e.g., $o_{t+1}=o_t$), so the first condition holds trivially.

\begin{theorem}[Markov property of fact state]
\emph{Under Assumption \ref{asm:sufficiency}, the stochastic process defined over the fact-centric states $s^{\text{fact}}_t$ is a MDP.}
\end{theorem}
\begin{proof}
We examine the transition probability for the fact-centric state. 
Given $s^{\text{fact}}_t=(Q,o_t,\mathcal{F}_t)$ and action $a_{t+1}$, the next feedback $o_{t+1}$ is generated by the environment, and the fact store is updated deterministically via $\mathcal{T}$.

\begin{align}
\resizebox{\columnwidth}{!}{$
\begin{aligned}
&\Pr(s^{\text{fact}}_{t+1} \,|\, s^{\text{fact}}_{0:t}, a_{1:t+1})
= \Pr(\mathcal{F}_{t+1}, o_{t+1} \,|\, s^{\text{fact}}_{0:t}, a_{1:t+1}) \\
&= \Pr(\mathcal{F}_{t+1} \,|\, o_{t+1}, s^{\text{fact}}_{0:t}, a_{1:t+1}) \cdot \Pr(o_{t+1} \,|\, s^{\text{fact}}_{0:t}, a_{1:t+1}) \\
&\stackrel{(i)}{=} \mathds{1}\!\left[\mathcal{F}_{t+1} = \mathcal{T}(\mathcal{F}_t, o_{t}, a_{t+1})\right] \cdot \Pr(o_{t+1} \,|\, s^{\text{fact}}_{0:t}, a_{1:t+1}) \\
&\stackrel{(ii)}{=} \mathds{1}\!\left[\mathcal{F}_{t+1} = \mathcal{T}(\mathcal{F}_t, o_{t}, a_{t+1})\right] \cdot \Pr(o_{t+1} \,|\, s^{\text{fact}}_t, a_{t+1}) \\
&= \Pr(s^{\text{fact}}_{t+1} \,|\, s^{\text{fact}}_t, a_{t+1}).
\end{aligned}
$}
\end{align}

Step (i) follows from the determinism of $\mathcal{T}$.
Step (ii) follows from Assumption \ref{asm:sufficiency}, which implies that $o_{t+1}$ depends on the past only through $s^{\text{fact}}_t=\Phi(h_t)$ and the current action.
\end{proof}

\begin{theorem}[optimal policy equivalence]
\emph{Under Assumption \ref{asm:sufficiency}, for any history $h$, the optimal value is preserved under the mapping $\Phi$: 
$V^*_{\text{hist}}(h) = V^*_{\text{fact}}(\Phi(h))$.}
\end{theorem}
\begin{proof}
Consider the Bellman optimality equation for the history-based agent:
\begin{equation}
    V^*_{\text{hist}}(h_t) = \max_{a} \mathbb{E}\!\left[ r(h_t,a) + \gamma V^*_{\text{hist}}(h_{t+1}) \,\middle|\, h_t, a \right].
\end{equation}

Under Assumption \ref{asm:sufficiency}, for any two histories $h_t,h_t'$ such that $\Phi(h_t)=\Phi(h_t')$, both the environment feedback distribution and the expected reward are identical:
\begin{align}
\resizebox{\columnwidth}{!}{$
\begin{aligned}
\Pr(o_{t+1} \,|\, h_t, a) &= \Pr(o_{t+1} \,|\, \Phi(h_t), a),\\
r(h_t, a) \coloneqq \mathbb{E}[R \,|\, h_t, a] &= \mathbb{E}[R \,|\, \Phi(h_t), a] \eqqcolon \bar{r}(\Phi(h_t), a).
\end{aligned}
$}
\end{align}

Therefore, the induced transition and reward model over the compressed state $s^{\text{fact}}_t=\Phi(h_t)$ is well-defined, and the optimal control problem can be equivalently written on this state space. In particular, restricting attention to policies that depend only on $\Phi(h)$ is without loss of optimality, yielding
$V^*_{\text{hist}}(h)=V^*_{\text{fact}}(\Phi(h))$.
\end{proof}

This analysis demonstrates that our method is theoretically grounded. By treating the search results $o_t$ as environment feedback driving the transitions of a fact-based state $s^{\text{fact}}_t=(Q,o_t,\mathcal{F}_t)$, we construct a valid MDP that allows for dense process supervision without violating the fundamental assumptions of reinforcement learning.

\section{Experiment Setup and Implementation}
\label{appx:experiment}

\subsection{RL Training Setup}
All training runs are conducted on 8× NVIDIA A100 GPUs, while evaluation is performed on 4× NVIDIA A100 GPUs.
We implement RL training based on the verl framework. Unless otherwise specified, we use the same hyperparameter configuration across all datasets and model scales. The RL hyperparameters are summarized in Table~\ref{tab:rl_training_details}.

\begin{table*}[t]
\centering
{\small
\begin{tabular}{lc|lc|lc}
\toprule
\textbf{Parameter} & \textbf{Value} & \textbf{Parameter} & \textbf{Value} & \textbf{Parameter} & \textbf{Value} \\
\midrule
Learning rate & 0.00001 & Training batch size & 128 & Global steps & 500 \\
Number of training epochs & 1 & Number of rollouts & 6 & Rollout temperature & 0.7 \\
KL loss coefficient & 0.001 & Clip ratio & 0.2 & Clustering similarity threshold & 0.95 \\
Parse error penalty & -\,0.1 & Duplicate search penalty & -\,0.01 \\
\bottomrule
\end{tabular}}
\caption{Hyperparameters setting of our RL training.}
\label{tab:rl_training_details}

\end{table*}

\subsection{Dataset Details}
\label{appx:dataset}
We evaluate our search agent on seven QA benchmarks covering both single- and multi-hop settings. 
The detailed statistics of the evaluation splits in our experiments are summarized in Table~\ref{tab:dataset_details}.

\begin{table*}[t]
\centering
{\small
\begin{tabular}{lccccccc}
\toprule
\textbf{Dataset} & \textbf{NQ} & \textbf{TriviaQA} & \textbf{PopQA} & \textbf{HotpotQA} & \textbf{2Wiki} & \textbf{MuSiQue} & \textbf{Bamboogle} \\
\midrule
Evaluation set size & 3,610 & 11,313 & 14,267 & 7,405 & 12,576 & 2,417 & 125 \\
\bottomrule
\end{tabular}}
\caption{The statistics of seven QA datasets.}
\label{tab:dataset_details}
\end{table*}

\section{Clustering Algorithm}
\label{appx:cluster}

In this section, we describe our clustering design. The key step is to decide whether two triples have the same meaning. 

\subsection{Two-Stage Fact Aggregation Strategy}

We use a simple two-stage test. First, we compute the cosine similarity  between the embeddings of two triples, and only pairs above a similarity threshold are passed to the next stage. 
Second, we apply conservative rule filtering, where all of the following conditions must be satisfied. 
We require negation consistency: if one relation is negated but the other is not, we never merge them. 
We also enforce numeric consistency: if either triple contains numbers, then both must contain numbers and the numeric parts must match exactly, since embeddings can be unreliable for fine-grained numeric differences. 
Additionally, we require the relation phrases themselves to be sufficiently similar, measured by either surface-form matching or by relation-embedding similarity above a threshold.

To handle possible subject-object inversion (e.g., ``A founded B'' vs. ``B was founded by A''), we test both configurations: the original triple pair and the pair with one triple's subject and object swapped. 
If either configuration passes all the above filtering rules, we treat the two triples as semantically equivalent and merge them.

\subsection{Calibration of Similarity Thresholds and Filtering Rules}

To ensure the robustness of our clustering algorithm, we conducted a rigorous calibration study. We utilized GPT-5.2 to generate positive and negative examples across 19 distinct scenarios, covering a wide range of linguistic variations including paraphrasing, negation, numeric alteration, and structural inversion.
Our analysis of the embedding similarity distribution (using the E5-base model) revealed the following patterns:
\begin{enumerate}
    \item \textbf{High similarity for semantic equivalence:} When two triples were semantically identical (e.g., simple paraphrases or formatting differences), their cosine similarity consistently exceeded 0.95.
    \item \textbf{False positives in specific error cases:} Despite the high discrimination power for general cases, we identified three specific categories where the embedding model assigned high similarity scores ($>0.95$) to semantically distinct triples:
    \begin{itemize}
        \item \textbf{Affirmation vs. Negation:} The model often failed to distinguish between a statement and its negation (e.g., ``A is B'' vs. ``A is not B'').
        \item \textbf{Numeric differences:} Slight variations in numbers (e.g., dates or quantities) resulted in indistinguishably high similarity scores, despite representing factual contradictions.
        \item \textbf{Relation ambiguity:} When the subject and object were identical but the relation changed slightly (e.g., ``investor in'' vs. ``CEO of''), embeddings alone were insufficient to separate the meanings reliably.
    \end{itemize}
\end{enumerate}

Based on these findings, we established a ``high-recall, strict-filtering'' strategy. We set a lenient similarity threshold ($\sigma_{sim}=0.95$) to capture diverse paraphrases, while implementing deterministic rules to filter out the false positives identified above (Negation Consistency, Numeric Consistency, and Relation-Specific checks).

Furthermore, to address the issue of Subject-Object Inversion, we introduced a bidirectional test. Since rigid subject/object matching rules might incorrectly exclude these semantically equivalent but structurally inverted pairs, we explicitly test both the original configuration and the swapped configuration. If either configuration passes the similarity threshold and the consistency rules, the triples are merged.

\begin{table*}[!t]
\centering
{\small
\begin{tabular}{lcccccccc}
\toprule
\textbf{Variant} 
& \textbf{NQ} 
& \textbf{TriviaQA} 
& \textbf{PopQA} 
& \textbf{HotpotQA} 
& \textbf{2Wiki} 
& \textbf{MuSiQue} 
& \textbf{Bamboogle} 
& \textbf{Avg.} \\
\midrule
w/ EM-only clustering & 43.5 & 65.8 & 44.0 & 45.2 & 48.8 & 19.1 & 42.4 & 48.9 \\
w/o reward redistribution & 42.1 & 63.9 & 42.1 & 38.3 & 43.7 & 14.0 & 38.4 & 45.4 \\
w/o process reward & 37.1 & 58.8 & 41.2 & 37.7 & 31.6 & 10.6 & 36.8 & 40.5 \\
\bottomrule
\end{tabular}}
\caption{Detailed ablation results of Qwen2.5-7B-Instruct on seven QA benchmarks (EM).}
\label{tab:ablation_detailed}
\end{table*}

\section{Complete Results for Ablation Study}
\label{appx:results}

In this section, we provide detailed results for the ablation study based on Qwen2.5-7B-Instruct. As shown in Table~\ref{tab:ablation_detailed}, we report per-dataset EM scores across seven QA benchmarks.

\begin{table}
\centering
{\small
\begin{tabular}{lc|lc}
\toprule
\textbf{Parameter} & \textbf{Value} & \textbf{Parameter} & \textbf{Value} \\
\midrule
Learning rate & 0.00002 & Global batch size & 64 \\
Warmup ratio & 0.05 & Total epochs & 2 \\
\bottomrule
\end{tabular}}
\caption{SFT cold-start training hyperparameters.}
\label{tab:sft_training_details}
\end{table}

\begin{table*}[t]
\centering
{\small
\begin{tabular}{lcccccccc}
\toprule
\textbf{Method} 
& \textbf{NQ} 
& \textbf{TriviaQA} 
& \textbf{PopQA} 
& \textbf{HotpotQA} 
& \textbf{2Wiki} 
& \textbf{MuSiQue} 
& \textbf{Bamboogle} 
& \textbf{Avg.} \\
\midrule
SFT (cold-start) & 32.9 & 50.9 & 35.1 & 26.0 & 20.6 & 8.6 & 24.8 & 32.3 \\
SFT + RL &  44.8&66.0  & 46.1 & 46.2 & 49.0 & 19.6 & 40.0 & 49.5 \\
\bottomrule
\end{tabular}
}
\caption{Cold-start results on seven QA benchmarks (EM) for Qwen2.5-7B-Instruct.}
\label{tab:coldstart_detailed}
\end{table*}

\begin{table*}[t]
\centering
{\small
\begin{tabular}{ccccccccc}
\toprule
\textbf{$\Omega$} 
& \textbf{NQ} 
& \textbf{TriviaQA} 
& \textbf{PopQA} 
& \textbf{HotpotQA} 
& \textbf{2Wiki} 
& \textbf{MuSiQue} 
& \textbf{Bamboogle} 
& \textbf{Avg.} \\
\midrule
0 & 37.1 & 58.8 & 41.2 & 37.7 & 31.6 & 10.6 & 36.8 & 40.5 \\
0.3 &  44.5&66.6 & 42.8 & 42.2 & 47.5 & 16.8 & 40.0 & 47.8 \\
0.5 &  46.3&69.4  & 46.7 & 44.6 & 51.5 & 19.4 & 44.0 & 51.2 \\
0.7 &  42.6&62.5  & 42.2 & 40.9 & 42.7 & 12.5 & 36.8 & 45.2 \\
\bottomrule
\end{tabular}
}
\caption{Sensitivity analysis of the aggregation weight $\Omega$.}
\label{tab:omega_sensitivity}
\end{table*}

\section{Cold-Start Analysis}
\label{appx:cold_start}

We implement cold start with a lightweight SFT stage to teach the Search--Assert--Answer workflow.
Specifically, we construct 600 trajectories via rejection sampling by rolling out the base model and filtering samples that are answer-correct, contain valid assert actions, and strictly follow the JSON schema.
We implement SFT using the verl framework, and summarize the SFT hyperparameters in Table~\ref{tab:sft_training_details}.
Complete results are reported in Table~\ref{tab:coldstart_detailed}.

From Table~\ref{tab:coldstart_detailed}, we observe that the SFT cold-start checkpoint does not consistently improve the initial performance of the agent. In fact, after SFT, the model exhibits a lower average EM compared to the naive initialization. 
Moreover, after subsequent RL training, the SFT-initialized agent still underperforms the counterpart trained from a naive initialization with the same RL procedure. These results suggest that, in our setting, the primary benefit of SFT cold-start lies in improving training stability and reducing early-stage degeneration, rather than increasing the final performance upper bound.

\begin{table}[h]
\centering
\small
\begin{tabular}{lccc}
\toprule
\textbf{Rollouts} ($n$) & \textbf{Time / Step} (s) & \textbf{Total Time} (h) & \textbf{EM} \\
\midrule
4 & 145.5 & 20.2 & 48.5 \\
6 (Default) & 231.1 & 32.2 & 51.2 \\
8 & 281.3 & 39.1 & 51.4 \\
\bottomrule
\end{tabular}
\caption{
Cost--quality tradeoff under different rollout group sizes.
}
\label{tab:rollout_tradeoff}
\end{table}

\section{Hyperparameter Sensitivity Analysis}
\label{appx:hyperparameter}

This section studies the sensitivity of our method to the aggregation weight $\Omega$, with results summarized in Table \ref{tab:omega_sensitivity}.
We observe that the final performance is sensitive to the choice of $\Omega$.
When $\Omega$ is relatively small, incorporating process-level rewards already yields substantial improvements over outcome-only supervision; for example, setting $\Omega=0.3$ improves the average EM by 7.3 points compared to pure outcome-based training.
However, as $\Omega$ becomes too large, the dominance of process rewards can negatively affect overall performance, leading to performance degradation.
Based on this analysis, we use $\Omega = 0.5$ as the default setting in all experiments.

\section{Training Efficiency and Computational Cost}
\label{appendix:efficiency}

To improve transparency regarding efficiency and computational overhead, we report additional latency and training statistics.
During inference, evaluated under concurrency 128 on 4$\times$A100 GPUs, FactAgent requires an average latency of 21.1 seconds per question and performs 3.8 reasoning steps on average.
During RL training, we observe that the average number of interaction steps naturally decreases as the policy improves, dropping from approximately 4.3 steps at initialization to 3.2 steps near convergence.
We further study the cost and performance tradeoff under different rollout group sizes in Fig.~\ref{tab:rollout_tradeoff}.

\begin{table}[t]
\centering
\small
\begin{tabular}{lccc}
\toprule
\textbf{Dataset} & \textbf{Top-1} & \textbf{Top-2} & \textbf{Top-3} \\
\midrule
NQ & 0.450 & 0.879 & 0.973 \\
HotpotQA & 0.414 & 0.822 & 0.936 \\
TriviaQA & 0.438 & 0.858 & 0.968 \\
\bottomrule
\end{tabular}
\caption{
Cumulative probability mass captured by top-ranked fact clusters.
}
\label{tab:cluster_mass}
\end{table}

\begin{table}[t]
\centering
\small
\begin{tabular}{lccc}
\toprule
\textbf{Dataset} & \textbf{Easy} & \textbf{Medium} & \textbf{Hard} \\
\midrule
NQ & 2.32 & 2.43 & 2.54 \\
HotpotQA & 2.43 & 2.51 & 2.70 \\
TriviaQA & 2.45 & 2.55 & 2.66 \\
\bottomrule
\end{tabular}
\caption{
Average number of unique fact clusters across difficulty groups.
}
\label{tab:difficulty_clusters}
\end{table}

\begin{table}
\centering
\small
\begin{tabular}{lccc}
\toprule
\textbf{Method} & \textbf{E-GRPO} & \textbf{VinePPO} & \textbf{FactAgent} \\
\midrule
NQ & 42.6 & 43.8 & 46.3 \\
TriviaQA & 62.8 & 62.1 & 69.4 \\
PopQA & 44.6 & 43.4 & 46.7 \\
HotpotQA & 35.1 & 34.4 & 44.6 \\
2Wiki & 44.2 & 47.1 & 51.5 \\
MuSiQue & 18.9 & 15.9 & 19.4 \\
Bamboogle & 40.8 & 40.0 & 44.0 \\
Avg. EM & 45.8 & 45.9 & 51.2 \\
\bottomrule
\end{tabular}
\caption{Comparison with representative process supervision approaches.}
\label{tab:related_comparison}
\end{table}

Increasing rollout count from 4 to 6 substantially improves performance, while increasing further to 8 only yields marginal gains.
To isolate FactAgent's additional overhead, we further measure fact collection and clustering cost. The additional clustering stage requires only 2.75 seconds on average per training step. Under rollout\_n=6, average training step time remains close to standard GRPO (231.1s vs.\ 225.2s).
These results suggest that FactAgent introduces only minor computational overhead while providing substantial performance gains.

\section{Utility Distribution Analysis}
\label{appendix:utility_distribution}

To better understand the structure of utility estimation, we rerun NQ, HotpotQA, and TriviaQA using multiple sampled trajectories and analyze fact-cluster distributions.

For each question, asserted facts are clustered and normalized into probability masses.
Across all datasets, utility distributions remain highly concentrated rather than uniformly sparse.
Table~\ref{tab:cluster_mass} reports cumulative cluster mass statistics.

The first three clusters already explain more than 93\% of total utility mass across all datasets, suggesting that successful reasoning trajectories rely primarily on a small number of highly informative evidence facts.
We further analyze question difficulty by grouping questions according to empirical success rates.

Hard questions require slightly more evidence aggregation, but utility distributions remain strongly concentrated.
These findings empirically support our design assumption that successful search trajectories depend on a small number of informative evidence facts rather than exponentially growing intermediate states.

\section{Comparison with Related Process Supervision Methods}
\label{appendix:related_work}

We compare FactAgent against representative process-supervision approaches.

In Table~\ref{tab:related_comparison}, compared with E-GRPO~\cite{egrpo}, which provides trajectory-level entity supervision, FactAgent performs step-level utility estimation without requiring additional annotation.
Compared with VinePPO~\cite{vineppo}, FactAgent avoids expensive per-step branching and Monte Carlo rollout expansion, while achieving stronger performance across benchmarks.

\section{System Prompts}
\label{appx:prompts}

Table \ref{tab:system_prompt} presents the full system prompt used for the search agent. We strictly enforce the output format to ensure the executability of the generated actions.

\section{AI Assistance Statement}
We used ChatGPT by OpenAI\footnote{\url{https://openai.com/chatgpt/}} and Gemini\footnote{\url{https://blog.google/products/gemini/gemini-3-gemini-app/}} for linguistic polishing and stylistic editing of the draft to enhance clarity and readability.

\begin{table*}[h]
\centering
\begin{tabular}{|p{0.9\textwidth}|}
\hline
\begin{lstlisting}[basicstyle=\small\ttfamily, breaklines=true, columns=fullflexible, showstringspaces=false]
**Response Format:**
1. **Reasoning**: First, output your thought process inside <think>...</think> tags. You can quote documents directly here without worrying about JSON format.
2. **Action**: Then, output exactly ONE valid JSON object inside a markdown block ```json ... ```.

**JSON Action Schema:**
- Search: {"action": "search", "parameters": {"query": "..."}}
- Assert: {"action": "assert", "parameters": {"triples": [{"subject":"...", "relation":"...", "object":"..."}], "evidence_summary":"..."}}
- Answer: {"action": "answer", "parameters": {"response": "..."}}

**Rules:**
- **<think> First**: Always think before you act.
- **Search**: Execute 'search' whenever you require external information. And the observation is correspond to the last query in the history.
- **Assert**: Extract facts from the Observation into the fact store. You may assert multiple facts in a single action.
- **Evidence**: Include an 'evidence_summary' that synthesizes useful details from the Observation and the asserted facts.
- **Answer**: Utilize both the current Observation and the fact store to derive your answer.
- **Conciseness**: The final answer must be short. Make the answer extremely concise (e.g., a single word, entity name, or short phrase). Do NOT use full sentences.

**Example (Multi-Triple Assert):**
User: Where was Richard Grant from and who hit him?
Assistant:
<think>
The observation says 'James Butler hit Richard Grant after a fight' and 'Richard Grant came from Harlem, New York'.
I need to extract two facts: one about the fight, and one about his origin.
</think>
```json
{
  "action": "assert",
  "parameters": {
    "triples": [
      {"subject": "James Butler", "relation": "hit", "object": "Richard Grant"},
      {"subject": "Richard Grant", "relation": "from", "object": "Harlem, New York"}
    ],
    "evidence_summary": "Observation states James Butler hit Richard Grant, who was from Harlem."
  }
}
\end{lstlisting} \\
\hline 
\end{tabular} 
\caption{The system prompt used for agents.}
\label{tab:system_prompt}
\end{table*}

\end{document}